\documentclass[a4paper]{article}
\title{Tight Bounds for Machine Unlearning via Differential Privacy}

\usepackage[T1]{fontenc}
\usepackage[utf8]{inputenc}
\author{
    Yiyang Huang\\
    {University of New South Wales}\\
    \texttt{sophia.huang1@student.unsw.edu.au}
    \and
    Cl\'ement L. Canonne\\
    {University of Sydney}\\
    \texttt{clement.canonne@sydney.edu.au}
}

\usepackage{styles}
\newcommand{\cW}{\mathcal{W}}
\newcommand{\cX}{\mathcal{X}}
\newcommand{\cT}{\mathcal{T}}

\begin{document}

\maketitle

\begin{abstract}
  We consider the formulation of ``machine unlearning'' of Sekhari, Acharya, Kamath, and Suresh (NeurIPS 2021), which formalizes the so-called "right to be forgotten" by requiring that a trained model, upon request, should be able to 'unlearn' a number of points from the training data, as if they had never been included in the first place. Sekhari et al. established some positive and negative results about the number of data points that can be successfully unlearnt by a trained model without impacting the model's accuracy (the ``deletion capacity''), showing that machine unlearning could be achieved by using differentially private (DP) algorithms. However, their results left open a gap between upper and lower bounds on the deletion capacity of these algorithms: our work fully closes this gap, obtaining tight bounds on the deletion capacity achievable by DP-based machine unlearning algorithms.
\end{abstract}

\section{Introduction}

Machine learning models trained on user data are now routinely used virtually everywhere, from recommendation systems to predictive models. In many cases, this user data itself includes some sensitive information (e.g., healthcare or race) or private aspects (customer habits, geographic data), sometimes even protected by law. To address this issue~--~that the models trained on sensitive datasets must not leak personal or private information~--~in a principled fashion, one of the leading frameworks is that of \emph{differential privacy} (DP)~\cite{DworkMNS06}, which has \emph{de facto} become the standard for privacy-preserving machine learning over the past decade.

At its core, DP requires that the output of a randomized algorithm $M$ not change drastically if one to modify one of the datapoints: that is, if $X,X'$ are two datasets only differing in \emph{one} user's data, then for all possible outputs $S$ of the algorithm one should have roughly the same probability of observing $S$ under both inputs:
\[
	\probaOf{M(X) \in S} \leq e^\eps \probaOf{M(X') \in S} + \delta
\]
where $\eps > 0$ and $\delta \in(0,1]$ quantify the privacy guarantee (the smaller values, the better the privacy; see~\cref{sec:prelims} for formal definitions). Intuitively, an algorithm $M$ being $(\eps,\delta)$-DP means that its output does not reveal much about any particular user's data, since the output would be nearly identical had this user's data been completely different. 

While the use of differential privacy can mitigate many privacy concerns, it does come with some limitations. The first is the overhead in brings: that is, ensuring differential privacy for a learning task typically incurs an overhead in the number of data points needed to achieve the same accuracy guarantee. Perhaps more importantly, DP does not solve all possible privacy concerns: even if a ML model is trained on a sensitive dataset in a differentially private way, the dataset may still be subject to some attacks~--~e.g., if the server where the training data is stored is itself compromised. Somewhat tautologically: DP  is not a silver bullet, and only provides meaningful guarantees against the threat models it was meant to address.

Another type of concerns focuses on the individual \emph{right to maintain control on one's own data}: broadly speaking, this is asking that each user can (under some reasonable circumstances) require that their personal data and information be removed from a company's collected data and trained models. This so-called ``right to be forgotten,'' which allow people to request that their data be deleted entirely from an ML system, has been passed into legislation or is considered in some form or another by various countries or entities, prominently the European Union's General Data Protection Regulation (GDPR), the California Privacy Rights Act (CCRA), Canada's proposed Consumer Privacy Protection Act (CPPA), and most recently in Australia~\cite{australiagdprlike}. 

However, translating this ``right to be forgotten'' into practice comes with a host of challenges, starting with ~\cite{GargGV2020} that provided a formal definitional framework using cryptographic concepts~--~which led to a new area of research in ML and computer science, that of \emph{machine unlearning}. A naive technical solution would be for a given company to keep the original training set at all times, and, upon a deletion request by a user, remove this user's data from the set before retraining the whole model on the result. This, of course, comes up with two major drawbacks: first, the cost to the company, in terms of time and computational resources, of retraining a large model on a regular basis. Second, the \emph{privacy cost}, as keeping the training set for an indefinite time in order to be able to handle the deletion requests leaves the door open to potential attacks and data breaches. Fortunately, there have been, over the past few years, a flurry of better (and more involved) approaches to machine unlearning, to handle deletion requests much more efficiently, and requiring to maintain much less of the training set (see, e.g.,~\cite{BourtouleCCJTZL21,NguyenHNLYN_survey_unlearning_22}, and related work below).

The above discussion, still, brings to light an important question: \emph{is machine unlearning, paradoxically, at odds with (differential) privacy? What is the connection between the two notions: are they complementary, or is there a trade-off between them?}

This is the main question this work sets out to address. Our starting point is the probabilistic definition of machine unlearning set forth by Sekhari, Acharya, Kamath, and Suresh~\cite{acharya_remember_unlearning_2021}, itself reminiscent of the definition of Differential Privacy (see Definition~\ref{def:unlearning} for the formal statement): a pair of algorithms $(A,\bar{A})$ is an \emph{$(\eps,\delta)$-unlearning algorithm} if (1) $A\colon \cX^\ast\to\cW$ is a (randomized) learning algorithm which, given a dataset $X\subseteq \cX^\ast$, outputs model parameters $A(X) \in \cW$; and (2)~$\bar{A}\colon \cX^\ast\times \cW\times \cT \to \cW$ which, on input a set of \emph{deletion requests} $U\subseteq X$, previous model parameters $w$, and some succinct additional ``side information'' $T(X)\in \cT$ about the original dataset, output updated model parameters $w' \in \cW$ from which the data from $U$ has been unlearned, that is, such that
\[
    \probaOf{\bar{A}(U, A(X), T(X)) \in W} \leq e^{\eps} \probaOf{\bar{A}(\emptyset, A(X \setminus U), T(X \setminus U)) \in W} + \delta
\]
and
\[
    \probaOf{\bar{A}(\emptyset, A(X \setminus U), T(X \setminus U)) \in W} \leq e^{\eps} \probaOf{\bar{A}(U, A(X), T(X)) \in W} + \delta
\]
for every possible set $W\subseteq \cW$ of model parameters. Loosely speaking, this requires that the outcomes of (a) training a model $M$ via $A$ on the dataset $X$ then unlearning some of the original training data $U\subseteq X$ from $M$ using $\bar{A}$, and (b) training a model $M'$ via $A$ directly on the dataset $X\setminus U$ then unlearning nothing via $\bar{A}$, be nearly indistinguishable. 

In their paper, Sekhari et al.~\cite{acharya_remember_unlearning_2021} focus on generalization guarantees of unlearning algorithm, i.e.,  what can be achieved by unlearning algorithms when focusing on population loss, namely, when aiming to minimize 
\[
    F(w) \coloneqq \mathbb{E}_{x\sim \mathcal{D}}[  f(w, x) ]
\]
given a prespecified loss function $f\colon \cW\times \cX \to \R$, where the expectation is over the draw of a new datapoint from the underlying distribution $p$ on the sample space. The quality of a learning algorithm $A$ is then measured by the expected excess risk 
\[
    R(f, A) \coloneqq \expect{ F(A(X)) - \inf_{w \in \cW} F(w) }
\]
where the expectation is taking over the random choice of a dataset $X\sim \mathcal{D}^n$ of size $n$, and the randomness of $A$ itself. The focus of~\cite{acharya_remember_unlearning_2021}, as is ours, is then to quantify the \emph{deletion capacity} achievable for $(\eps,\delta)$-unlearning given a prespecified loss function, that is, the maximum number of data points one can ask to be forgotten (maximum size of the subset $U$) before the excess risk increases by more than some threshold (see Definition~\ref{eq:del:capacity}).

In their paper,~\cite{acharya_remember_unlearning_2021} draw a connection between DP learning algorithms and unlearning ones, showing that DP learning algorithms do provide \emph{some} unlearning guarantees out-of-the-box, and that one can achieve non-trivial unlearning guarantees for convex loss functions by leveraging the literature on differentially private optimization and learning. One of their main results is showing that these DP-based unlearning algorithms, which crucially \emph{do not rely on any side information} (the additional input $T(X)\in \cT$ provided to the unlearning algorithm $\bar{A}$) can handle strictly fewer deletion requests than general unlearning algorithms which \emph{do} rely on such side information.

Their results, however, do not fully characterize the deletion capacity of these ``DP-based'' machine unlearning algorithms, leaving a significant gap between their upper and lower bounds. We argue that fully understanding this quantity is crucial, as DP-based unlearning algorithms are \emph{exactly} those for which there is no conflict between the two notions of DP and unlearning~--~\emph{instead, this class of algorithms is the one for which they work hand in hand.} This is in contrast to the more general unlearning algorithms relying on maintaining and storing side information about the training set, as this side information can make their deployment susceptible to privacy breaches.

\subsection{Our contributions}

The main contribution of our paper is a tight bound on the ``amount of unlearning'' achievable by \emph{any} machine unlearning algorithm which does not rely on side information. For the sake of exposition, we state in this section informal versions of our results.

\begin{theorem}[Unlearning For Convex Loss Functions (Informal; see Theorems~\ref{theo:unlearning:convex:ub} and~\ref{theo:unlearning:convex:lb})]
    \label{theo:main:informal}
    Let $f\colon\cW\times \cX\to\R$ be a $1$-Lipschitz convex loss function, where $\cW\subseteq\R^d$ is the parameter space. There exists an $(\eps,\delta)$-machine unlearning algorithm which, trained on a dataset $S\subseteq \cX^n$, does not store \emph{any} side information about the training set besides the learned model, and can unlearn up to
    \[
    	m = O\!\mleft( \frac{n\eps\alpha}{\sqrt{d\log(1/\delta)}} \mright)
    \]
    datapoints without incurring excess population risk greater than $\alpha$. Moreover, this is tight: there exists a $1$-Lipschitz \emph{linear} loss function such that no machine unlearning algorithm can unlearn $\Omega( \frac{n\eps\alpha}{\sqrt{d\log(1/\delta)}} )$ data points without excess population risk $\alpha$, unless it stores side information.
\end{theorem}

Our techniques also allow us to easily establish the analogue for \emph{strongly} convex optimization:
\begin{theorem}[Unlearning For Strongly Convex Loss Functions (Informal)]
    Let $f\colon\cW\times \cX\to\R$ be a $1$-Lipschitz \emph{strongly} convex loss function. There exists an $(\eps,\delta)$-machine unlearning algorithm which, trained on a dataset $S\subseteq \cX^n$, does not store \emph{any} side information about the training set besides the learned model, and can unlearn up to
    \[
    	m = O\!\mleft( \frac{n\eps\sqrt{\alpha}}{\sqrt{d\log(1/\delta)}} \mright)
    \]
    datapoints without incurring excess population risk greater than $\alpha$. Moreover, this is tight.
    \label{theo:main:informal:sc}
\end{theorem}

We note that, prior to our work, only bounds for the convex loss function case were known, with an upper bound of $m = \tilde{O}( n\eps\alpha/\sqrt{d\log(e^\eps/\delta)})$ (loose by polylogarithmic factors for $\eps = O(1)$, as well as an $1/\sqrt{\eps}$ factor for $\eps \gg 1$) and a limited lower bound stating that $m \geq 1$ is only possible if $n\eps/\sqrt{d} = \Omega(1)$. We point out that while the high-privacy ($\eps \ll 1$) is sought as desirable, it is common in real-life, deployed systems to use much larger values of $\eps$, typically $\eps \gg 1$. We also provide analogous bounds for the case of \emph{pure} unlearning, i.e., $\delta=0$ (\cref{theo:pure:unlearning:ub,theo:pure:unlearning:lb}).

Our next contribution, motivated by the similarity of the formalisation of machine unlearning (without side information) and that of differential privacy, is to establish the analogue of key properties of DP for machine unlearning, namely, \emph{post-processing} and \emph{composition} of machine unlearning algorithms. To do so, we first identify a natural property of machine unlearning algorithms, which, when satisfied, will allow for the composition properties:

\begin{assumption}[Unlearning Laziness]
    \label{assumption:unlearning-chaining}
    An unlearning algorithm $(\bar{A}, A)$ is said to be \emph{lazy} if, when provided with an \emph{empty} set of deletion requests, the unlearning algorithm $\bar{A}$ does not update the model. That is, $\bar{A}(\emptyset, A(X), T(X)) = A(X)$ for all datasets $X$.
\end{assumption}

We again emphasize that this laziness property is not only intuitive, it is also satisfied by several existing unlearning algorithms, and in particular those proposed in~\cite{acharya_remember_unlearning_2021}.

\begin{theorem}[Post-processing of  unlearning]
    Let $(\bar{A}, A)$ be an $(\eps,\delta)$-unlearning algorithm taking no side information. Let $f\colon \cW \to \cW $ be an arbitrary (possibly randomized) function. Then $(f \circ \bar{A}, A)$ is also an $(\eps,\delta)$-unlearning algorithm.
\end{theorem}

Under our laziness assumption, we also establish the following:
\begin{theorem}[Chaining of unlearning]
    Let $(\bar{A}, {A})$ be a \emph{lazy} $(\eps, \delta)$-unlearning algorithm taking no side information, and able to handle up to $m$ deletion requests. Then, the algorithm which uses $(\bar{A}, {A})$ to sequentially unlearn $k$ disjoint deletion requests $U_1, \ldots, U_k \subseteq X$ such that $|\cup_{i} U_i| \leq m$, outputting
        \[
            \bar{A}(U_k, \ldots, \bar{A}(U_1, A(X))\ldots)
        \]
    is an $(\eps', \delta')$-unlearning algorithm, with $\eps' = k \eps$ and $\delta' = \delta\cdot \frac{e^{k\eps}-1}{e^\eps-1} = O(k\delta)$ (for $k = O(1/\eps)$).
\end{theorem}

and, finally,
\begin{theorem}[Advanced grouposition of unlearning]
        Let $(\bar{A}_1, A), \ldots, (\bar{A}_k, A)$ be \emph{lazy} $(\eps, \delta)$-unlearning (with common learning algorithm $A$) taking no side information, and define the grouposition of those unlearning algorithms, $\tilde{A}$ as
        \[
            \tilde{A}(U, A(X)) = f\mleft( \bar{A}_1(U, A(X)), \ldots, \bar{A}_k(U, A(X)) \mright)\,.
        \]
        where $f\colon \cW^k\to\cW$ is any (possibly randomized) function. 
        Then, for every $\delta'>0$, $( \tilde{A},A)$ is an $(\eps', \delta')$-unlearning taking no side information, where $\eps' = \frac{k}{2}\eps^2 + \eps\sqrt{2k\ln{(1/\delta')}}$.
        \label{thm:unlearning-grouposition}
    \end{theorem}

\subsection{Related work}
    Albeit recent, the field of machine unlearning has already received considerable attention from the ML community since introduced in \cite{CaoYang2015}. There emerged an array of studies and papers focusing on practical solutions and their empirical performance. We focus in this section on the works most closely related to ours, mostly theoretical.

    Literature in machine unlearning that relates to differential privacy branches to two: (1) models are prone to attacks when attackers have access to both before and after version when the deletion requests are processed, and (2) the conceptual similarity of machine unlearning and differential privacy.
    

    The original, stringent definition of unlearning requires $\eps = 0$ (full deletion of the user's data, as if it had never been included in the training set in the first place) in contrast to differential privacy that allows $\eps > 0$, leaving a possibility for ``memorization.'' To relax this definition, \cite{Ginart2019MakingAF} proposed the probabilistic version of unlearning. The formalization of machine unlearning problem was proposed by \cite{GargGV2020}, adapted from \cite{Ginart2019MakingAF}'s definition, as a \emph{deletion compliance framework}. In their work, they claimed the free unlearning effects of differential privacy. However, the main difference with our work is that \cite{GargGV2020} relied on the statistical distance in quantifying error (privacy) whereas we adapt the usual $(\eps, \delta)$-DP notion. In addition, their work does not investigate the actual learning performance and quantifies the size of samples (deletion capacity) that we can unlearn.

    After the formalization of the unlearning problem, a major line of work that involves both unlearning with differential privacy emerged to investigate adversarial scenarios. For instance, \cite{SBBFZ2020} discussed differential privacy as a defense mechanism to combat inference and reconstruction attacks, similarly for \cite{CZWBHZ2021} and \cite{GGMV2022}.
    
    There has also been considerable past theoretical work on machine unlearning and DP that are related to our study, yet most works center upon \textit{empirical risk minimization} of unlearning algorithms (e.g. \cite{GuoGHM20, IzzoUnlearning20}, \cite{DBLP:conf/alt/Neel0S21}, \cite{DBLP:conf/colt/UllahM0RA21} etc.), which seeks to find an approximate minimizer on the remaining dataset after deletion. Closest to our work is the recent paper of \cite{acharya_remember_unlearning_2021}, which formulated the notion of machine unlearning used in our paper and focused on \textit{population loss minimization} of approximating unlearning algorithm (i.e., allowing $\eps>0$). Their objectives, however, were somewhat orthogonal to ours, as they focused for a large part on minimizing the space requirements for the side information $T(X)$ provided to the unlearning algorithm (while our paper focuses on algorithms which do \emph{not} rely on any such side information, prone to potential privacy leaks). While their work, to motivate this focus, established partial bounds on the deletion capacity of unlearning algorithm that do not take in extra statistics, these bounds were not tight, and one of our main contributions is closing this gap.
     
    Following~\cite{acharya_remember_unlearning_2021}, the notion of \emph{online} unlearning algorithm~--~which receive the deletion requests sequentially~--~was put forward and studied in \cite{SuriyakumarW22}, again with memory efficiency with respect to the side information in mind; however, their primary focus is on the empirical performance of unlearning algorithm.
    
    Another work closely to ours is the notion of \emph{certified data removal} proposed by \cite{GuoGHM20}. The main difference between $(\eps, \delta)$-certified removal and the definition from \cite{acharya_remember_unlearning_2021} is that, in the former, the unlearning mechanism requires access not only to the samples to be deleted (the set $U\subseteq X$), but also to the full original training set $X$: this is exactly the type of constraints our work seeks to avoid, due to the risk of data breach this entails.

\subsection{Organization of the paper}
    We first provide the necessary background and notion on differential privacy, learning, and the formulation of machine unlearning used throughout the paper in Section~\ref{sec:prelims}. We then provide a detailed outline of the proof of our main result, Theorem~\ref{theo:main:informal}, in Section~\ref{sec:main}, before concluding with a discussion of results and future work in Section~\ref{sec:discussion}.
    Finally, proofs are presented in Appendix.

\section{Preliminaries}
	\label{sec:prelims}
 In this section, we recall some notions and results we will extensively rely on in our proofs and theorems, starting with differential privacy.
    \subsection{Differential Privacy}
        \begin{definition}[(Central) Differential Privacy (DP)]
            Fix $\eps>0$ and $\delta\in[0,1]$. An algorithm $M\colon \mathcal{X}^n \to \mathcal{Y}$ satisfies \emph{$(\eps, \delta)$-differential privacy} (DP) if for every pair of neighboring datasets $X, X'$, and every (measurable) subset $S \subseteq \mathcal{Y}$:
            \[
                \probaOf{M(X) \in S} \leq e^\eps \probaOf{M(X') \in S} + \delta.
            \]
            We further say that $M$ satisfies \emph{pure} differential privacy ($\eps$-DP) if $\delta = 0$, otherwise it is \emph{approximate} differential privacy.
        \end{definition}

        We now recall another notion of differential privacy in terms of Renyi Divergence, from~\cite{bun_steinke_zcdp}.
        \begin{definition}[Zero-Concentrated Differential Privacy (zCDP)]
            A randomized algorithm $M: \mathcal{X}^n \to \mathcal{Y}$ satisfies $(\xi, \rho)$-zCDP if for every neighboring datasets (differing on a single entry) $X, X' \in \mathcal{X}^n$, and $\forall \alpha \in (1, \infty)$:
            \[
                \operatorname{D}_{\alpha}(M(X) \| M(X')) \leq \xi + \rho \alpha
            \]
            where $\operatorname{D}$ is the $\alpha$-Renyi divergence between distributions of $M(X)$ and $M(X')$. We say that $M$ is $\rho$-zCDP when $\xi = 0$.
        \end{definition}

        We use the following group privacy property of zCDP in the proof later.

        \begin{proposition}[$k$-distance group privacy of $\rho$-zCDP {\cite[Proposition 1.9]{bun_steinke_zcdp}}]
            Let $M: \mathcal{X}^n \to \mathcal{Y}$ satisfy $\rho$-zCDP. Then, $M$ is $(k^2 \rho)$-zCDP for every $X, X' \in \mathcal{X}^n$ that differs in at most $k$ entries.
            \label{proposition:zcdp-group-privacy}
        \end{proposition}

    \subsection{Learning}    
        We also will require some definitions on learning, specifically with respect to minimizing population loss. Fix any loss function $f\colon \mathcal{W} \times \mathcal{X}$, where $\cW$ is the (model) parameter space and $\cX$ is the sample space. Then, the generalization loss is defined as
        \[
            F(w) \coloneqq \mathbb{E}_{x\sim D}[  f(w, x) ]
        \]
        in which the expectation is over the distribution of $x$ (one sample) and $w$ is the learning output.
        
        Then, let $F^* = \inf_{w \in \mathcal{W}} F(w)$ be the minimized population risk and $w^\ast$ is the corresponding minimizer. Define learning algorithm $A: \mathcal{X}^n \to \mathcal{W}$ that takes in dataset $S \in \mathcal{X}^n$ and returns hypothesis $w \coloneqq A(S) \in \mathcal{W}$. Then we can define excess risk as:
        \[
            R(f, A) \coloneqq \expect{F(A(S)) - F^*}
        \]
        where the expectation is over the randomness of $A$ and $S$.

        Finally, we could define the sample complexity as following (\cite[Definition 1]{acharya_remember_unlearning_2021}), which is analogous to deletion capacity, in which will be stated later.
        
        \begin{definition}[Sample complexity of learning]
            The $\alpha$-sample complexity of a problem is defined as:
            \[
                n(\alpha) \coloneqq \min \mleft\{ n \mid \exists A \text{ s.t. } \expect{F(A(S))} - F^* \leq \alpha, \: \forall \mathcal{D}\mright\}
            \]
        \end{definition}

    \subsection{Unlearning}
    As previously discussed, we rely on the definition of unlearning proposed in by \cite{acharya_remember_unlearning_2021}, and maintain same notation. Note that $T(S)$ denotes the data statistics (which could be the entire dataset $S$ or any form of statistic) available to $\bar{A}$. 
        \begin{definition}[$(\eps, \delta)$-unlearning]
            \label{def:unlearning}
            For all $S$ of size $n$ and delete requests $U \subseteq S$ such that $|U| \leq m$, and $W \subseteq \mathcal{W}$, a learning algorithm $A$ and an unlearning algorithm $\bar{A}$ is $(\eps, \delta)$-unlearning if:
            \[
                \probaOf{\bar{A}(U, A(S), T(S)) \in W} \leq e^{\eps} \probaOf{\bar{A}(\emptyset, A(S \setminus U), T(S \setminus U)) \in W} + \delta
            \]
            and
            \[
                \probaOf{\bar{A}(\emptyset, A(S \setminus U), T(S \setminus U)) \in W} \leq e^{\eps} \probaOf{\bar{A}(U, A(S), T(S)) \in W} + \delta,
            \]
        \end{definition}
         Our results will be phrased in terms of the deletion capacity, which captures the number of deletion requests an unlearning algorithm can handle before seeing a noticeable drop in its output's accuracy:
        \begin{definition}[Deletion Capacity]
        \label{eq:del:capacity}
            Let $\eps, \delta > 0$, $S$ be a dataset of size $n$ drawn i.i.d.\ from $\mathcal{D}$ and let $\ell(w,z)$ be a loss function. For a pair of learning and unlearning algorithm $A, \bar{A}$ that are $(\eps, \delta)$-unlearning, the deletion capacity $m^{A, \bar{A}}_{\eps, \delta}$ is defined as the maximum size of deletions requests set $|U|$ that we can unlearn without doing worse in excess population risk than $\alpha$:
            \[
                m^{A, \bar{A}}_{\eps, \delta}(\alpha) := \max \{ m \mid \expect{\max_{U \subseteq S: |U| \leq m} F(\bar{A}(U, A(S), T(S))) - F^*} \leq \alpha \}
            \]
            where $F^* = \min_{w \in \mathcal{W}} F(w)$.
        \end{definition}

\section{Main result}
\label{sec:main}

In this section, we provide a detailed outline of our main result on unlearning for convex loss functions, Theorem~\ref{theo:main:informal}, for which we prove the upper and lower bounds separately.

Note that all proofs in this paper are presented in the Appendix.

\begin{theorem}[Deletion capacity from unlearning via DP, Lower Bound]
    \label{theo:unlearning:convex:ub}
        Suppose $\cW\subseteq \R^d$, and fix any Lipschitz convex loss function. Then there exists a lazy ($\eps,\delta)$-unlearning algorithm $(\bar{A}, A)$, where $\bar{A}$ has the form $\bar{A}(U, A(S), T(S)) := A(S)$ (and thus, in particular, takes no side information) with deletion capacity
            \[
                m_{\eps, \delta}^{A, \bar{A}}(\alpha) \geq \Omega\mleft( \frac{\eps n \alpha}{\sqrt{d \log{(1/\delta)}}} \mright)
            \]
        where the constant in the $\Omega(\cdot)$ only depends on the properties of the loss function.
\end{theorem}

\begin{theorem}[Deletion capacity from unlearning via DP, Upper Bound]
    \label{theo:unlearning:convex:lb}
            There exists a Lipschitz convex loss function (indeed, \emph{linear}) for which any $(\eps,\delta)$-unlearning algorithm $(\bar{A}, A)$ which takes no side information must have deletion capacity
            \[
                m^{A, \bar{A}}_{\eps, \delta}(\alpha) \leq O\mleft( \frac{\eps n \alpha}{\sqrt{d \log{(1/\delta)}}} \mright)\,.
            \]
\end{theorem}

We note that the proof of Theorem~\ref{theo:main:informal:sc} follows from a very similar argument; we refer the reader to the Appendix for the convex version.

We also present analogous deletion capacity bounds for the $\eps$-unlearning case:

\begin{theorem}[Deletion capacity from unlearning via DP, Lower Bound]
    \label{theo:pure:unlearning:lb}
    Suppose $\cW \subseteq \R^d$, and fix any $L$-Lipschitz (strongly) convex loss function. Then there exists a lazy $(\eps,\delta)$-unlearning algorithm $(\bar{A}, A)$, where $\bar{A}$ has the form $\bar{A}(U, A(S), T(S)) := A(S)$ (and thus, in particular, which takes no side information) with deletion capacity
    \[
        m^{A, \bar{A}}_{\eps, \delta}(\alpha) \geq \bigOmega{\frac{\eps n \alpha^2}{d}}.
    \]
\end{theorem}

\begin{theorem}[Deletion capacity from unlearning via DP, Upper Bound]
    \label{theo:pure:unlearning:ub}
    There exists a Lipschitz convex loss function for which any $\eps$-unlearning algorithm $(\bar{A}, A)$ which takes no side information must have deletion capacity
    \[
        m^{A, \bar{A}}_{\eps, \delta}(\alpha) \leq O\mleft( \frac{\eps n\alpha}{d}\mright)\,.
    \]
\end{theorem}

\section{Discussion and future work}
    \label{sec:discussion}
    Our work fully characterizes the deletion capacity of any unlearning algorithm $(\bar{A},A)$ minimizing population risk under both convex and strongly convex loss functions, when only given the model parameters (output of the learning algorithm) and the set of deletion requests. This restriction, namely that the unlearning algorithm does not rely on any additional side information, is motivated by the potential privacy risks storing (non-private) side information can entail.

    We hope our work will lead to further study of the interplay between differential privacy and machine unlearning, and to additional study of ``DP-like'' properties of machine unlearning, such as the postprocessing and composition properties our present work identified. In view of the myriad applications these properties have had in privacy-preserving algorithm design, we believe that their analogue for machine unlearning will prove very useful.  

    We leave for future work the question of whether variants of the standard threat model for differential privacy (specifically, pan-privacy, or privacy under continual observation) could have implications for machine unlearning in an online setting where deletion requests come sequentially.


\paragraph{One final remark.} Before concluding, we believe it is important to reiterate a key point about our work. One could view our main results as focusing on what can be achieved by unlearning algorithms ``which, upon a deletion request, do not do anything,'' since they leverage the guarantees already provided by differential privacy. However, this does \emph{not} mean that the algorithm ``does nothing'' \emph{overall}: instead, the point here is that the algorithms considered already satisfy the very stringent notion of DP, and as such already paid some utility cost to provide this guarantee: as a result, \emph{having paid that price}, they benefit from some ``unlearning bonus'' for free.

Put differently, the aim of this paper is not to promote or justify when it comes to unlearning, but instead to characterize exactly how much unlearning guarantees come ``for free'' if one decides or already has to offer the strong guarantee of differential privacy. Thus, our aim is not to discourage unlearning-only solutions when DP is not required; but instead, by understanding the interplay between DP and unlearning, to show that the joint (and seemingly daunting) requirement of both differential privacy and right to be forgotten requirement is more affordable than it seems. 


\printbibliography

\appendix
\label{appendix}

\section{Proof of Theorem 3.1 (Lower Bound)}
\label{appendix:thm:unlearning:lb:proof}

\begin{theorem}[Deletion capacity from unlearning via DP, Lower Bound (Theorem 3.1 in Submission)]
    Suppose $\cW\subseteq \R^d$, and fix any Lipschitz convex loss function. Then there exists a lazy $(\eps,\delta)$-unlearning algorithm $(\bar{A}, A)$, where $\bar{A}$ has the form $\bar{A}(U, A(S), T(S)) := A(S)$ (and thus, in particular, takes no side information) with deletion capacity
        \[
            m_{\eps, \delta}^{A, \bar{A}}(\alpha) \geq \Omega\mleft( \frac{\eps n \alpha}{\sqrt{d \log{(1/\delta)}}} \mright)
        \]
    where the constant in the $\Omega(\cdot)$ only depends on the properties of the loss function.
\end{theorem}


\begin{lemma}[zCDP mini-batch noisy SGD {\cite{vitaly_private_sco_linear_2020}}]
    \label{lemma:zcdp-sco-bound}
    Fix any $L$-Lipschitz convex loss function over a convex subset $\mathcal{B}$ of $\mathbb{R}^d$ of diameter $D$. Then there exists an algorithm $A$ which satisfies $(\rho^2/2)$-zCDP with excess population loss:
    \begin{equation*}
        \expect{F(\theta) - \min_{\theta \in \mathcal{B}} F(\theta)} \leq O\mleft(DL \cdot \mleft( \frac{1}{\sqrt{n}} + \frac{\sqrt{d}}{\rho n} \mright) \mright)
    \end{equation*}
    where the expectation is taken over the randomness of $A$.
\end{lemma}

\begin{proof}[Proof of \cref{theo:unlearning:convex:ub}]
    The proof follows the same setting as in \cite{acharya_remember_unlearning_2021}. The main change is that we apply group privacy bounds in terms of zCDP instead of the standard DP guarantee provided by \cite[Theorem 3.2]{bassily_sco_2019}.

    We first establish a tighter bound for algorithm that achieves $m$-entries group privacy via \cref{lemma:zcdp-sco-bound}. \cite{vitaly_private_sco_linear_2020} provides a zCDP version of \cite[Theorem 3.2]{bassily_sco_2019} with $\rho^2/2$-zCDP, hence by group privacy, we yield $\frac{m^2\rho^2}{2}$-zCDP by \cref{proposition:zcdp-group-privacy} for neighboring datasets differing in $m$ entries. Then, translating $\frac{m^2\rho^2}{2}$-zCDP to $(\eps, \delta)$-DP yields $\eps = O\mleft({m\rho \sqrt{\log{(1/\delta)}}}\mright)$. 

    By the above discussion, using this zCDP-private learning algorithm with $\rho=\Theta\mleft(\frac{\eps}{m \sqrt{\ln{(1/\delta)}}}\mright)$, we get an excess population loss bounded by
        \begin{equation}
            \label{eq:guarantee:learning}
            O\mleft( DL\mleft( \frac{1}{\sqrt{n}} + \frac{m\sqrt{d \ln{(1/\delta)}}}{\eps n} \mright)\mright)
        \end{equation}

    It only remains to show how the claimed deletion capacity bound follows from  this excess population risk guarantee. Construct, as discussed earlier, an unlearning algorithm $\bar{A}$ that returns the input without making any changes (and in particular does not require any additional statistics $T(S)$, and satisfies the laziness assumption).
            Since $A$ is $(\eps, \delta)$-DP, for any set $U \subseteq S, |U| = m$, and $W\subseteq \cW$,
            \begin{align*}
                \probaOf{A(S) \in W} &\leq e^\eps \probaOf{A(S') \in W} + \delta\\
                \probaOf{A(S') \in W} &\leq e^\eps \probaOf{A(S) \in W} + \delta
            \end{align*}.
    But since $\bar{A}(U, A(S)) = A(S)$, this readily yields, letting $S' := S \setminus U$:  
    \begin{align*}
        \probaOf{\bar{A}(U, A(S)) \in W} &\leq e^\eps \probaOf{\bar{A}(\emptyset, A(S')) \in W} + \delta\\
        \probaOf{\bar{A}(\emptyset, A(S')) \in W} &\leq e^\eps \probaOf{\bar{A}(U, A(S))\in W} + \delta
    \end{align*}
    which implies that $(A, \bar{A})$ is indeed $(\eps, \delta)$-unlearning for $U$ of size (up to) $m$.

    Recalling the definition of deletion capacity, we finally deduce from~\eqref{eq:guarantee:learning} the deletion capacity with excess population risk less than $\alpha$:
    \[
        m^{A, \bar{A}}_{\eps, \delta}(\alpha) \geq m = \Omega\mleft( \frac{\eps n \alpha}{\sqrt{d \ln{(1/\delta)}}} \mright)
    \]
    where the $O(\cdot)$ hides constant factors depending only on the loss function (namely, the Lipschitz function $L$, and the diameter $D$).
\end{proof}

\section{Proof of Theorem 3.2 (Upper Bound)}
\label{appendix:thm:unlearning:ub:proof}

\begin{theorem}[Deletion capacity from unlearning via DP, Upper Bound (Theorem 3.3 in Submission)]
    There exists a Lipschitz convex loss function (indeed, \emph{linear}) for which any $(\eps,\delta)$-unlearning algorithm $(\bar{A}, A)$ which takes no side information must have deletion capacity
    \[
        m^{A, \bar{A}}_{\eps, \delta}(\alpha) \leq O\mleft( \frac{\eps n \alpha}{\sqrt{d \log{(1/\delta)}}} \mright)\,.
    \]
\end{theorem}
\begin{proof}[Proof of \cref{theo:unlearning:convex:lb}]
    We will consider the following linear (and therefore convex and Lipschitz) loss function $\mathcal{L}(\theta, S)$:
             \begin{equation}
                \mathcal{L}(\theta, S) := -\langle \theta, \sum^n_{i=1}x_i \rangle
            \end{equation}
            for dataset $S$ of $n$ points $x_1,\dots, x_n \in\{-\frac{1}{\sqrt{d}}, \frac{1}{\sqrt{d}}\}^d$.      
            We also define the 1-way marginal query, i.e. average, as:
            \begin{equation}
                \label{eq:q}
                    q(S) := \frac{1}{n}\sum^n_{i=1}x_i\,.
            \end{equation}
    To establish our deletion capacity lower bound with respect to this loss function, we will proceed in three stages: the first, relatively standard, is to relate population loss (what we are interested in) to \emph{empirical} loss~--~which allows us to focus on the existence of a ``hard dataset.'' The second step is then to establish a sample complexity lower bound on the empirical risk (for this loss function) of any $(\eps,\delta)$-DP algorithm, via a reduction to differentially private computing of 1-marginals. This step is similar to the one underlying the (weaker) lower bound of~\cite{acharya_remember_unlearning_2021} (itself relying on an argument of~\cite{bassily_sco_2019}), although a more careful choice of building blocks for the reduction already allows us to obtain an improvement by logarithmic factors.

    Third, lift this DP lower bound to a stronger lower bound for DP with respect to edit distance $m$. This step is quite novel, as it morally corresponds to establishing the converse of the grouposition property of differential privacy (for our specific setting), a converse which does \emph{not} hold in general. Our argument, relatively simple, will crucially rely on the linearity of our loss function.

    We omit the details of the first step (reduction from population to empirical loss) in this detailed outline, as it is quite standard. For the second step, our starting point is the following lower bound of Steinke and Ullman:
    \begin{theorem}[Lower bound for one-way marginals {\cite[Main Theorem]{steinke_between_2016}}] \label{theorem:main_theorem_bst14}
                For every $\eps \in (0, 1)$, every function $\delta = \delta(n)$ such that $\delta \geq 2^{-o(n)}$ and $\delta \leq 1/n^{1+\Omega({1}})$, and for every $\alpha \leq 1/10$, if $A: \{\pm 1\}^{n \times d} \to [\pm 1]^d$ is $(\eps, \delta)$-differentially private and $\expect{\|\mathcal{A}(S) - q(S)\|_1} \leq \alpha d$ (i.e., with average-case accuracy $\alpha$) for all $S \in \{\pm 1\}^{n\times d}$, then we must have
                \[
                    n \geq \Omega\mleft({\frac{\sqrt{d \ln{(1/\delta)}}}{\eps \alpha}}\mright).
                \]
    \end{theorem} 
    
    Using this lower bound as a blackbox, we then can adapt the argument of~\cite[Lemma~5.1, Part 2]{bassily_erm_2014} to obtain the following stronger result:
    \begin{lemma}[Lower bound for Privately Computing 1-way Marginals] \label{lemma5.1-part2}
        Let $n, d \in \mathbb{N}, \eps > 0, 2^{-o{n}} \leq \delta(n) \leq 1/n^{1 + \Omega(1)}$. For all $\alpha \leq 1/10$, if $\mathcal{A}$ is $(\eps, \delta)$-differentially private. Then, for $S \subseteq \{\pm \frac{1}{\sqrt{d}}\}^{n \times d}$, one must have
        \[
            \expect{\|\mathcal{A}(S) - q(S)\|_2} = \min\mleft(\alpha, \Omega\mleft(\frac{\sqrt{d\ln{(1/\delta)}}}{n\eps}\mright)\mright)\,,
        \]
        where $q(S) = \frac{1}{n}\sum^n_{i=1}x_i$ as before. Moreover, this still holds under the assumption that $\|q(S)\|_2 \in [\frac{M-1}{n},\frac{M+1}{n}]$, where $M = \Omega(\min(n\alpha, \frac{\sqrt{d \ln{(1/\delta)}}}{\eps }))$.
    \end{lemma}
    \begin{proof}[Proof of \cref{lemma5.1-part2}]
        Our proof follows the same outline as in \cite{bassily_erm_2014}, but using the result of~\cref{theorem:main_theorem_bst14} as a black-box instead of the packing argument of~\cite{bassily_erm_2014}. Before doing so, we have to translate the result from Theorem \ref{theorem:main_theorem_bst14} into our setting, and handle the slightly different choice of parameterization ($\{\pm 1\}^d$ instead of $\{\pm 1/\sqrt{d}\}^d$).

        Let $n_\alpha \coloneqq C\cdot \frac{\sqrt{d \ln{(1/\delta)}}}{\eps \alpha}$, where $C>0$ is (strictly smaller than) the constant hidden in the $\Omega(\cdot)$ of~\cref{theorem:main_theorem_bst14}. By contradiction, suppose that, for some $n \leq n_\alpha$, we have an $(\eps, \delta)$-differentially private algorithm $\mathcal{A}$ that takes in a dataset $S \subseteq \{\pm \frac{1}{\sqrt{d}}\}^{n \times d}$ and outputs an estimate $\mathcal{A}(S)$ of $q(S)$ with expected $L_2$ error $\alpha$. Rescaling, we get that the algorithm $\mathcal{A}'$ which, on input $S'\subseteq \{\pm 1\}^{n \times d}$, computes $S \coloneqq S'/\sqrt{d}\subseteq \{\pm \frac{1}{\sqrt{d}}\}^{n \times d}$ and outputs $\sqrt{d}\cdot \mathcal{A}(S)$ is (1)~$(\eps, \delta)$-DP by post-processing, and (2)~since $q$ is linear, has error related to that of $\mathcal{A}$ by
        \begin{equation}
            \label{eq:bound:l2:error:Aprime}
            \expect{\|\mathcal{A}'(S') - q(S')\|_2}
            = \sqrt{d}\cdot \expect{\|\mathcal{A}(S) - q(S)\|_2}
            \leq \sqrt{d} \cdot  \alpha
        \end{equation}
        However, by~\cref{theorem:main_theorem_bst14}, $\mathcal{A}'$ must have expected $L_1$ error at least $\alpha d$ since $n \leq n_\alpha$. By Cauchy--Schwarz,
        \[
            \alpha d < \expect{\|\mathcal{A}'(S') - q(S')\|_1}
            \stackrel{\rm{}CS}{\leq} \sqrt{d}\cdot \expect{\|\mathcal{A}'(S') - q(S')\|_2}
            \stackrel{\eqref{eq:bound:l2:error:Aprime}}{\leq} \sqrt{d} \cdot (\alpha\sqrt{d}) = \alpha d
        \]
        leading to a contradiction. So for $n \leq n_\alpha$, any $(\eps, \delta)$-DP algorithm to estimate $q$ must have expected $L_2$ error at least $\alpha$, i.e., $\expect{\|\mathcal{A}(S) - q(S)\|_2} \geq \alpha$. Further, one can see by inspection of the proof of~\cref{theorem:main_theorem_bst14} that $\|q(S)\|_2$ satisfies the assumption in the "Moreover."\smallskip

        Now, for $n \geq n_\alpha$ (assume, for simplicity and without loss of generality, that $n-n_\alpha$ is even), we use a padding argument to establish the other part of the bound.                 Let $\mathcal{A}$ be any $(\eps, \delta)$-differentially private algorithm for answering $q$ on datasets of size $n$. Suppose for the sake of contradiction, that $\mathcal{A}$ satisfies
        \begin{equation}
            \label{eq:assumption:for:contradiction:padding}
            \expect{\|\mathcal{A}(S) - q(S)\|_2}
            < \frac{n_\alpha}{n}\cdot \alpha
        \end{equation}
        for every dataset $S$ of size $n$.

        Fix an arbitrary point $\mathbf{c} \in \{\pm 1 / \sqrt{d}\}^d$. Given any dataset $S=(x^{(1)},\dots,x^{(n_\alpha)})\in \mleft\{\pm 1\mright\}^{d\times n_\alpha}$ of size $n_\alpha$, we construct $\hat{S}$ of size $n$ as follows. Its first $n_\alpha$ entries are $x^{(1)},\dots,x^{(n_\alpha)}$; then for the remaining $n - n_\alpha$, we have (1) the first $\lceil \frac{n - n_\alpha}{2} \rceil$ (i.e. the first half) of those entries are all copies of $\mathbf{c}$, and (2) the remaining $\lfloor \frac{n - n_\alpha}{2} \rfloor$ are copies of $-\mathbf{c}$.

        Note that we have
        \[
            q(\hat{S}) =
            \frac{n_\alpha}{n}q(S)
        \]
        for every $S$, and in particular  $\|q(\hat{S})\|_2$ satisfies the assumption in the "Moreover."

        Now, we define an algorithm $\hat{\mathcal{A}}$ for approximating $q$ on datasets of size $n_\alpha$ as follows. On input $S\in \mleft\{\pm 1\mright\}^{d\times n_\alpha}$, $\hat{\mathcal{A}}$:
        \begin{enumerate}
            \item Computes $\hat{S}\in \mleft\{\pm 1\mright\}^{d\times n}$ as above
            \item Outputs $\frac{n}{n_\alpha}\mathcal{A}(\hat{S})$
        \end{enumerate}

        Since $\mathcal{A}$ is already differentially private, $\hat{\mathcal{A}}$ is also $(\eps, \delta)$-DP due to the post-processing property of differential privacy. Moreover,
        \[
            \expect{\|\hat{\mathcal{A}}(S) - q(S)\|_2}
            = \expect{\mleft\|\frac{n}{n_\alpha}\mathcal{A}(\hat{S}) - \frac{n}{n_\alpha}q(\hat{S})\mright\|_2}
            = \frac{n}{n_\alpha}\expect{\mleft\|\mathcal{A}(\hat{S}) - q(\hat{S})\mright\|_2}
            \stackrel{\eqref{eq:assumption:for:contradiction:padding}}{<} \frac{n}{n_\alpha}\cdot \frac{n_\alpha}{n}\alpha = \alpha
        \]
        and so $\hat{\mathcal{A}}$ achieves expected error strictly smaller than $\alpha$ on datasets of size $n_\alpha$; which contradicts the first part of the lower bound we already established. So for $n > n_\alpha$, any $(\eps, \delta)$-DP algorithm to estimate $q$ must have expected $L_2$ error at least $\frac{n_\alpha}{n}\cdot \alpha = C\cdot \frac{\sqrt{d\ln(1/\delta)}}{n\eps}$.\smallskip

        Finally, we we have shown that for every $n$ and every $\eps > 0$, there is a constant $C>0$ such that every $(\eps, \delta)$-differentially private algorithm $\mathcal{A}$ answering the linear query $q$ must have, on some dataset $S$ of size $n$, expected $L_2$ error at least
        \[
            \expect{\|\mathcal{A}(S) - q(S)\|_2} = \min\mleft( \alpha,C\cdot \frac{\sqrt{d\ln(1/\delta)}}{n\eps} \mright).
        \]
        proving the lemma.
    \end{proof}

    Combining the above with the argument strategy of~\cite[Theorem 5.3]{bassily_erm_2014} finally yields the main lemma for the second step of our proof for \cref{theo:unlearning:convex:ub}:

    \begin{lemma}[Lower bound on empirical loss of $(\eps, \delta)$-DP algorithms]
        Let $n, d \in \mathbb{N}, \eps > 0$, and $\delta = o(1/n)$. For every $(\eps, \delta)$-differentially private algorithm with output $\theta^{priv}$, there is a dataset $S = \{x_1, \ldots, x_n\} \subseteq \{-\frac{1}{\sqrt{d}}, \frac{1}{\sqrt{d}}\}^d$ such that
        \[
            \expect{\mathcal{L}(\theta^{priv}, S) - \mathcal{L}(\theta^*, S)} = \min\mleft( \alpha^2, \Omega\mleft(\frac{d\log(1/\delta)}{n^2\eps^2}\mright) \mright)
        \]
        where $\theta^* := \frac{\sum^n_{i=1}x_i}{\| \sum^n_{i=1}x_i \|_2}$ is the minimizer of $\mathcal{L}(\theta, S) := - \langle \theta, \frac{1}{n}\sum^n_{i=1}x_i \rangle$ (which is linear and, as such, Lipschitz and convex).
        \label{theorem:erm-dp} 
    \end{lemma}
    \begin{proof}[Proof of \cref{theorem:erm-dp}]
        This proof follows the same structure as that of \cite[Theorem~5.3]{bassily_erm_2014} but adapt the bound in terms of expectation.
        
        First, observe that for any $\theta \in \mathbb{B}$ and dataset $S$ we have:
        \[
        \mathcal{L}(\theta, S) - \mathcal{L}(\theta^*, S)
            = \frac{1}{2}\|q(S)\|_2 \| \theta - \theta^* \|^2_2,
        \]
        since $\|\theta - \theta^* \|^2_2 = \|\theta^*\|^2_2 + \|\theta\|^2_2 - 2\langle \theta, \theta^* \rangle = 2(1 - \langle \theta, \theta^* \rangle)$ using the fact that $\theta^*, \theta \in \mathbb{B}$ have $\|\theta\|_2, \|\theta^\ast\|_2 = 1$.

        Suppose that there is an $(\eps, \delta)$-differentially private algorithm $\mathcal{A}$ that outputs $\theta^{priv}$ such that, for every dataset $S \subseteq \{-\frac{1}{\sqrt{d}}, \frac{1}{\sqrt{d}}\}^d$, we have:
        \[
            \expect{\mathcal{L}(\theta^{priv}, S) - \mathcal{L}(\theta^*, S)} \leq \Delta
        \]
        for a sufficiently small constant $C>0$, and some $\Delta\geq 0$. We will prove a lower bound on $\Delta$.  
        To do so, recall $q(S) = \theta^* \cdot \|q(S)\|_2$; and that the lower bound from~\cref{lemma5.1-part2} still holds when the dataset $S$ is promised to be such that $q(S) \in [(M\pm 1)/n]$, for $M=\Theta(\min(n\alpha,\sqrt{d\log(1/\delta)}/\eps)$. 

Consider the algorithm (private by post-processing) $\mathcal{A}$ which outputs $\mathcal{A}(S) = \frac{M}{n}\theta^{priv}$. Then, for any dataset $S$ such that $\|\sum_{i=1}^n x_i \|_2 \in [M-1,M+1]$,\[
\expect{\|\mathcal{A}(S) - q(S) \|_2 } \leq
\expect{\|\mathcal{A}(S) - q(S) \|_2^2 }^{1/2}
= \expect{\|\frac{M}{n}\theta^{priv} - q(S) \|_2^2 }^{1/2}\,.
\]
On the other hand,
        \begin{align*}
            \expect{\|\frac{M}{n}\theta^{priv} - q(S) \|_2^2 }
            &\leq 2\mleft( \expect{\|q(S)\|_2^2\|\theta^{priv} - \theta^*\|_2^2} + \expect{\|\frac{M}{n}\theta^{priv} - \|q(S)\|_2\theta^{priv} \|_2^2} \mright)\\
            &= 4\|q(S)\|_2\expect{\mathcal{L}(\theta^{priv}, S) - \mathcal{L}(\theta^*, S)} + 2\left(\frac{M}{n} - \|q(S)\|_2 \right)^2 \\
           &\leq \frac{4(M+1)}{n}\expect{\mathcal{L}(\theta^{priv}, S) - \mathcal{L}(\theta^*, S)} + \frac{2}{n^2}  \tag{as $n\|q(S)\|_2 \in [M-1,M+1]$}\\
           &\leq \frac{4(M+1)\Delta}{n} + \frac{2}{n^2}
        \end{align*}
        By \cref{lemma5.1-part2}, we know that $\expect{\|\mathcal{A}(S) - q(S)\|_2} = \min\mleft( \alpha,C\cdot \frac{\sqrt{d\ln(1/\delta)}}{n\eps} \mright)$, for some absolute constant $C>0$, in the worst case.
        Hence, we must have
        \[
            \frac{\Delta \cdot M}{n} \geq \min\mleft( \alpha^2,\frac{d\ln(1/\delta)}{n^2\eps^2} \mright);
        \]
        recalling the setting of $M$, we get $\expect{\mathcal{L}(\theta^{priv}, S) - \mathcal{L}(\theta^*, S)} = \min\mleft( \alpha, \Omega\mleft(\sqrt{\frac{d\ln(1/\delta)}{n\eps}}\mright) \mright)$.
    \end{proof}

    The above lemma establishes a lower bound on the empirical loss of any $(\eps, \delta)$-differentially private algorithm. To derive from this our claimed lower bound on unlearning algorithms, we need to introduce a dependence on $m$, the deletion capacity (i.e., number of points to unlearn). This is done in the last (third) step of our argument, via a reduction which establishes a (restricted) converse to the grouposition property of DP.

    Recall that an algorithm $M \colon \mathcal{X}^n \to \mathcal{Y}$ satisfies $(\eps, \delta)$-DP for edit distance $m$ if for every pair of neighboring datasets \emph{$X, X'$ that differ in up to $m$ entries}, and every $S \subseteq \mathcal{Y}$:
    \[
        \probaOf{M(X) \in S} \leq e^\eps \probaOf{M(X') \in S} + \delta.
    \]
    We apply this $m$-edit distance $(\eps, \delta)$-DP on \cref{theorem:erm-dp} by a reduction that shows: for any differentially private algorithm with respect to edit distance at most $m$ must incur an empirical loss given by \cref{theorem:erm-dp}.

    \begin{lemma}
        \label{lemma:reduction:lb:medit}
        Suppose there exists an $m$-edit distance $(\eps, \delta)$-DP algorithm $\mathcal{M}$ that takes in a dataset $S$ of size $n$ to approximate $q(S)$ (as defined in~\eqref{eq:q}), with empirical loss $\gamma$. Then, we can construct a $1$-edit distance (i.e., standard) $(\eps, \delta)$-DP algorithm $\mathcal{M}'$ that, on input a dataset $S'$ of $N = n/m$ data points, approximates $q(S')$ to error $\gamma$.
    \end{lemma}

    \begin{proof}[Proof of \cref{lemma:reduction:lb:medit}]
        The reduction is quite simple: given $\mathcal{M}$, construct $\mathcal{M}'$ as follows for $N = \frac{n}{m}$ inputs:
        \[
            \mathcal{M}'(x_1, \ldots, x_N) = \mathcal{M}(\underbrace{x_1, \ldots, x_1,}_m \underbrace{x_2, \ldots, x_2}_m, \ldots, \underbrace{x_N, \ldots, x_N}_m)\,.
        \]
        We immediately have that $\mathcal{M}'$ is $(\eps,\delta)$-DP for the usual $1$-edit distance between datasets, since $\mathcal{M}$ is DP with respect to edit distance $m$. The sample complexity and error bound then follows direction from $n = N \times m$, where $n \geq N, N \in \mathbb{N}, m \geq 1$, and the fact that $q(x_1,\dots, x_N) = q(x_1,\dots, x_1,x_2,\dots, x_2, \dots, x_N, \dots, x_N)$ due to the definition of $q$.
    \end{proof}  
    Combining Lemma~\ref{lemma:reduction:lb:medit} with \cref{theorem:erm-dp}, we get that any $m$-edit distance $(\eps, \delta)$-DP algorithm $\mathcal{M}$ approximating $q$ on datasets of size $n=N\times m$ must have error $\gamma$ at least
    \[
    \min\mleft( \alpha, \Omega\mleft(\frac{\sqrt{d\log(1/\delta)}}{N\eps}\mright) \mright)
    = \min\mleft( \alpha, \Omega\mleft(\frac{m\sqrt{d\log(1/\delta)}}{n\eps}\mright) \mright)
    \]
    which, reorganising the terms and recalling the definition of deletion capacity, yields the claimed bound on $m^{A, \bar{A}}_{\eps, \delta}$, and hence completes the proof for \cref{theo:unlearning:convex:lb}.
\end{proof}

The proof of Theorem 1.2 (the strongly convex case), restated below, is analogous to those of \cref{theo:unlearning:convex:ub,theo:unlearning:convex:lb}, but using~\cite[Theorem~5.1]{vitaly_private_sco_linear_2020} for the upper bound (instead of~\cref{lemma:zcdp-sco-bound}) and ~\cite[Theorem~5.2]{steinke_between_2016} for the lower bound (instead of~\cref{theorem:main_theorem_bst14}).

\begin{theorem}[Unlearning For Strongly Convex Loss Functions (Theorem 1.2, restated)]
    Let $f\colon\cW\times \cX\to\R$ be a $1$-Lipschitz \emph{strongly} convex loss function. There exists an $(\eps,\delta)$-machine unlearning algorithm which, trained on a dataset $S\subseteq \cX^n$, does not store \emph{any} side information about the training set besides the learned model, and can unlearn up to
    \[
    	m = O\!\mleft( \frac{n\eps\sqrt{\alpha}}{\sqrt{d\log(1/\delta)}} \mright)
    \]
    datapoints without incurring excess population risk greater than $\alpha$. Moreover, this is tight.
\end{theorem}

\section{Proof of Theorems~3.3 and 3.4}
    \begin{lemma}[Lower bound on empirical loss of $\eps$-DP algorithms]
        Let $n, d \in \mathbb{N}, \eps > 0$. For every $\eps$-differentially private algorithm with output $\theta^{priv}$, there is a dataset $S = \{x_1, \ldots, x_n\} \subseteq \{-\frac{1}{\sqrt{d}}, \frac{1}{\sqrt{d}}\}^d$ such that
        \[
            \expect{\mathcal{L}(\theta^{priv}, S) - \mathcal{L}(\theta^*, S)} = \min\mleft(1, \Omega\mleft(\frac{d}{n\eps}\mright) \mright)
        \]
        where $\theta^* := \frac{\sum^n_{i=1}x_i}{\| \sum^n_{i=1}x_i \|_2}$ is the minimizer of $\mathcal{L}(\theta, S) := - \langle \theta, \frac{1}{n}\sum^n_{i=1}x_i \rangle$ (which is linear and, as such, Lipschitz and convex).
        \label{lemmma:lb:erm:eps-dp}
    \end{lemma}
    \begin{proof}
        We translate \cite[Theorem 5.2]{bassily_erm_2014} in to expected value for the sake of coherence throughout this paper. From \cite[Theorem 5.2]{bassily_erm_2014}, we have that, for some constant $c>0$, the statement can be expressed as:
        \[
            \probaOf{\mathcal{L}(\theta^{priv}, S) - \mathcal{L}(\theta^*, S) \geq c \cdot \min(n, d/\eps)} \geq 1/2
        \]
        Then since $\mathcal{L}(\theta^{priv}, S) - \mathcal{L}(\theta^*, S) \geq 0$,
        \[
            \expect{\mathcal{L}(\theta^{priv}, S) - \mathcal{L}(\theta^*, S)}
            \geq c \cdot \min(n, d/\eps)\cdot \probaOf{\mathcal{L}(\theta^{priv}, S) - \mathcal{L}(\theta^*, S) \geq c \cdot \min(n, d/\eps)}
        \]
        we get
        \[
            \expect{\mathcal{L}(\theta^{priv}, S) - \mathcal{L}(\theta^*, S)}
            \geq \frac{c}{2} \cdot \min(1, \frac{d}{n\eps})
        \]
        by re-normalizing the loss function by $\frac{1}{n}$.
    \end{proof}
    
    \begin{theorem}[Upper bound on number of removable samples for $\eps$-unlearning algorithm]
        Suppose $\cW \subseteq \R^d$, and fix any Lipschitz convex loss function. Then there exists a lazy $\eps$-unlearning algorithm $(\bar{A}, A)$, where $\bar{A}$ has the form $\bar{A}(U, A(S), T(S)) := A(S)$ (and thus, in particular, takes no side information) with deletion capacity
        \[
            m^{A, \bar{A}}_{\eps, \delta}(\alpha) \leq O\mleft( \frac{n\alpha\eps}{d}\mright)\,.
        \]
    \end{theorem}
    \begin{proof}
        Using what we have from \cref{lemmma:lb:erm:eps-dp}, we can apply the similar reductions as in \cref{appendix:thm:unlearning:ub:proof} (from ERM to SCO, and from 1-distance DP to $m$-distance DP through \cref{lemma:reduction:lb:medit}.

        Hence the lower bound on the expected excess risk of generalization error is given by, $N = \frac{n}{m}$, where $N$ is the total number of samples in the reduction in \cref{lemma:reduction:lb:medit} with $\delta = 0$:
        \[
            \min \mleft(1, \bigOmega{\frac{d}{N\eps}} \mright)
            = \min \mleft(1, \bigOmega{\frac{md}{n\eps}} \mright)
        \]
        then by rearranging with the expected excess risk bound $\alpha$, we get the desired deletion capacity $m$.
    \end{proof}

    \begin{theorem}[Lower bound on number of removable samples for $\eps$-unlearning algorithm]
        Suppose $\cW \subseteq \R^d$, and fix any $L$-Lipschitz $\Delta$-(strongly) convex loss function. Then there exists a lazy $(\eps,\delta)$-unlearning algorithm $(\bar{A}, A)$, where $\bar{A}$ has the form $\bar{A}(U, A(S), T(S)) := A(S)$ (and thus, in particular, takes no side information) with deletion capacity
        \[
            m^{A, \bar{A}}_{\eps, \delta}(\alpha) \geq \bigOmega{\frac{\eps n \alpha^2 \Delta}{dL}}.
        \]
    \end{theorem}
    \begin{proof}[Proof Sketch]
        By adapting the proof of given in~\cite[Section~5.4]{ShalevShwartz2009StochasticCO}, which relies on an expected guarantee before applying Markov's inequality, we get the following variant of~\cite[Theorem~F.1]{bassily_erm_2014}:
        \begin{theorem}
            Let $\ell$ be an $L$-Lipschitz, $\Delta$-strongly convex loss function. Then, for any parameter $\theta$ (possibly depending on $\mathcal{D}$), the following holds over the randomness of sampling the data set $\mathcal{D}$:
            \[
                \expect{\mathrm{ExcessRisk}(\theta)}
                \leq 
                \sqrt{\frac{2L^2}{n\Delta}}\expect{\sqrt{(\mathcal{L}(\theta, \mathcal{D})-\mathcal{L}(\theta^\ast(\mathcal{D}), \mathcal{D}))}}+ \frac{4L^2}{n\Delta}
            \]
        \end{theorem}
        In particular, applying this to the output $\hat{\theta}=\hat{\theta}(\mathcal{D})$ of any randomized algorithm, and using Jensen's inequality, we get
        \begin{equation}
                \expect{\mathrm{ExcessRisk}(\hat{\theta})}
                \leq 
                \sqrt{\frac{2L^2}{n\Delta}}\sqrt{\expect{\mathcal{L}(\hat{\theta}, \mathcal{D})-\mathcal{L}(\theta^\ast(\mathcal{D}), \mathcal{D}))}}+ \frac{4L^2}{n\Delta}
        \end{equation}
        where the expectation is over the draw of the dataset $\mathcal{D}$ and the randomness of the algorithm.

        Then, applying this theorem with the upper bound of the $\eps$-DP algorithm from \cite{bassily_erm_2014}, we have
        \begin{equation}
            \expect{\mathrm{ExcessRisk}(\hat{\theta})}
            \leq C\cdot \mleft( 
            \sqrt{\frac{L^2}{n\Delta}}\sqrt{\frac{d}{\eps}}+ \frac{L^2}{n\Delta} \mright)
        \end{equation}
        for some absolute constant $C>0$. 

        By using \cref{lemma:reduction:lb:medit} again, we have:
        \begin{equation}
            \expect{\mathrm{ExcessRisk}(\hat{\theta})}
            \leq 
            C\cdot \mleft( \sqrt{\frac{mdL^2}{\eps n \Delta}} + \frac{mL^2}{n\Delta} \mright)
        \end{equation}
        which gives the desired lower bound.
    \end{proof}

\section{Proof of $(\eps, \delta)$-unlearning properties}

    The laziness (Assumption~\ref{assumption:unlearning-chaining}) is essential for the proof, and a natural requirement for practical applications. 

    \begin{theorem}[Post-processing of  unlearning (Theorem 1.4 in Submission)]
        Let $(\bar{A}, A)$ be an $(\eps,\delta)$-unlearning algorithm taking no side information. Let $f\colon \cW \to \cW $ be an arbitrary (possibly randomized) function. Then $(f \circ \bar{A}, A)$ is also an $(\eps,\delta)$-unlearning algorithm.
    \end{theorem}
    \begin{proof}
        The proof follows exactly same as post-processing property of differential privacy. We consider the case that $f$ is a deterministic function here without loss of generality.
        
        Let $T = \{r \in \R^d \mid f(r) \in \mathcal{Y}\}$ and $\mathcal{Y} \subseteq \R^d$. Consider for any $\mathcal{Y} \subseteq \R^d$:
        \begin{align*}
            \probaOf{f(\bar{A}(A(S), U)) \in \mathcal{Y}}
            &= \probaOf{\bar{A}(A(S), U) \in T}\\
            &\leq e^\eps \probaOf{\bar{A}(A(S), U) \in T} + \delta\\
            &= e^\eps \probaOf{f(\bar{A}(A(S), U)) \in \mathcal{Y}} + \delta
        \end{align*}
    \end{proof}

    Under our laziness assumption, we can establish bounds on applying unlearning algorithm repeatedly when the overall deletion requests is within the deletion capacity:
    \begin{theorem}[Chaining of unlearning (Theorem 1.5 in Submission)]
        Let $(\bar{A}, {A})$ be a \emph{lazy} $(\eps, \delta)$-unlearning algorithm taking no side information, and able to handle up to $m$ deletion requests. Then, the algorithm which uses $(\bar{A}, {A})$ to sequentially unlearn $k$ disjoint deletion requests $U_1, \ldots, U_k \subseteq X$ such that $|\cup_{i} U_i| \leq m$, outputting
        \[
            \bar{A}(U_k, \ldots, \bar{A}(U_1, A(X))\ldots)
        \]
        is an $(\eps', \delta')$-unlearning algorithm, with $\eps' = k \eps$ and $\delta' = \delta\cdot \frac{e^{k\eps}-1}{e^\eps-1} = O(k\delta)$ (for $k = O(1/\eps)$).
    \end{theorem}
    \begin{proof}
        We proceed by induction on $n\geq 1$. Given a pair of $(\eps, \delta)$-unlearning algorithm $(\bar{A}, A)$ and deletion requests $D_1, \ldots, D_n \subseteq S \in \R^{n \times d}$ such that $|\cup_{i} D_i| \leq m^{\bar{A}, A}_{\eps, \delta}$ with $D_i \cap D_j = \empty, \forall i \neq j$ for $i, j\in [n]$.
    
        \par (1)~For $n = 1$:
        \[
            \probaOf{\bar{A}(A(S), D_1) \in T} \leq e^{n\eps} \probaOf{\bar{A}(A(S \setminus D_1), \emptyset)} + \delta
        \]
        by the definition of $(\eps, \delta)$-unlearning. Hence the case $n=1$ holds.

        \par (2)~Assume $n = k$ is true:
        \begin{equation}
            \probaOf{\bar{A}(\ldots \bar{A}(A(S), D_1), \ldots, D_k) \in T} \leq e^{k\eps} \probaOf{\bar{A}(A(S \setminus \bar{D}_k), \emptyset)} + \sum^{k-1}_{i=0} e^{i\eps} \cdot \delta
            \label{eq:chaining:n=k}
        \end{equation}

        \par (3)~Then for $n=k+1$:
        \begin{align*}
            \probaOf{\bar{A}(\ldots \bar{A}(A(S), D_1), \ldots, D_{k+1}) \in T}
            &\stackrel{(\ref{eq:chaining:n=k})}{\leq} e^{k\eps} \probaOf{\bar{A}(\bar{A}(A(S \setminus \bar{D}_k), \emptyset), D_{k+1})} + \sum^{k-1}_{i=0} e^{i\eps} \cdot \delta\\
            &= e^{k\eps} \probaOf{\bar{A}(A(S \setminus \bar{D}_k), D_{k+1})} + \sum^{k-1}_{i=0} e^{i\eps} \cdot \delta\\
            &\leq e^{(k+1)\eps}\probaOf{\bar{A}(A(S \setminus \bar{D}_{k+1}), \emptyset) \in T} + \sum^{(k+1) - 1}_{i=0} e^{i\eps} \cdot \delta
        \end{align*}
        where the first and third inequality result from the definition of $(\eps, \delta)$-unlearning and the second equality is due to Laziness Assumption \ref{assumption:unlearning-chaining}.

        Hence, by induction, the claim holds for all $n \in \mathbb{N}$.
    \end{proof}

    \begin{theorem}[Advanced grouposition of unlearning]
        Let $(\bar{A}_1, A), \ldots, (\bar{A}_k, A)$ be \emph{lazy} $(\eps, \delta)$-unlearning (with common learning algorithm $A$) taking no side information, and define the composition of those unlearning algorithms, $\tilde{A}$ as
        \[
            \tilde{A}(U, A(X)) = f\mleft( \bar{A}_1(U, A(X)), \ldots, \bar{A}_k(U, A(X)) \mright)\,.
        \]
        where $f\colon \cW^k\to\cW$ is any (possibly randomized) function. 
        Then, for every $\delta'>0$, $( \tilde{A},A)$ is an $(\eps', \delta')$-unlearning taking no side information, where $\eps' = \frac{k}{2}\eps^2 + \eps\sqrt{2k\ln{(1/\delta')}}$.
    \end{theorem}
    \begin{proof}
        The proof follows the same argument as in \cite[Lemma 2.4]{vadhan_complexity_dp}. We consider the case of $\delta > 0$ only as the $\delta=0$ is same with the pure DP proof.

        Fix two datasets, $S$ (original dataset) and $S' \coloneqq S \setminus U$ (``forgotten dataset'') where $U$ is the set of delete requests with $|U| \leq m^{\bar{A}, A}_{\eps, \delta}$. Note that $S, S'$ differs in $m$ entries.
        
        For an output $y = (y_1, \ldots, y_k) \in \mathcal{Y}$, define ``memory'' loss (which is just privacy loss in differential privacy) to be:
        \[
            \mathcal{L}_{\mathcal{A}}^{S \to S'}(y) = \ln \frac{\probaOf{\mathcal{A}(A(S), U) = y}}{\probaOf{\mathcal{A}(A(S'), \emptyset) = y}}
        \]
        where $|\mathcal{L}_{\mathcal{A}}^{S \to S'}(y)| \leq \eps$.

        Then, by \cite[Lemma 1.5]{vadhan_complexity_dp} we know that $\bar{A}_i(A(S), U), \bar{A}_i(A(S'), \emptyset)$ are $(\eps, \delta)$-indistinguishable, hence there are events $E = E_1 \wedge \ldots \wedge E_k, E' = E'_1 \wedge \ldots \wedge E'_k$  such that w.p. at least $1 - k\delta$ by, for all $y_i, i \in [k]$,
        \begin{align*}
            \expect{\mathcal{L}_{\mathcal{A}}^{S \to S'}(y)}
            &= \expect{\ln \frac{\probaCond{\mathcal{A}(A(S), U) = y}{E}}{\probaCond{\mathcal{A}(A(S'), \emptyset) = y}{E'}}}\\
            &= \sum^k_{i=1} \expect{\ln \mleft(\frac{\probaCond{\bar{A}_i(A(S), U) = y}{E_i}}{\probaCond{\bar{A}_i(A(S'), \emptyset) = y}{E'_i}}\mright)}\\
            &= \sum^k_{i=1} \expect{\mathcal{L}^{S \to S'}_{\bar{A}_i}(y)}
        \end{align*}
        where we observe that the expectation of the loss is just KL-divergence between the distributions of $\bar{A}_i(A(S), U)$ and $\bar{A}_i(A(S'), \emptyset)$ conditioned on $E$ and $E'$. Hence:
        \[
            \expect{\mathcal{L}_{\mathcal{A}}^{S \to S'}(y)}
            = \sum^k_{i=1} \text{D}_{\text{KL}} (\bar{A}_i(A(S), U) \| \bar{A}_i(A(S'), \emptyset))
            \leq \frac{k}{2}\eps^2
        \]
        where the inequality is a result from \cite[Proposition 3.3]{bun_steinke_zcdp} when $\alpha = 1$. This proposition is applicable because the conditional distribution of $\bar{A}_i$ is $(\eps, \delta)$-indistinguishable, which shares the max-divergence definition.

        Then by Hoeffding's inequality where the loss is bounded by $[-\eps, \eps]$, with probability at least $1 - \delta'$, we have:
        \begin{align*}
            \exp{\mleft(-\frac{t^2}{2k\eps^2}\mright)}
            &\geq \probaOf{\mathcal{L}_{\mathcal{A}}^{S \to S'}(y) > \expect{\mathcal{L}_{\mathcal{A}}^{S \to S'}(y)} + t}\\
            &\geq \probaOf{\mathcal{L}_{\mathcal{A}}^{S \to S'}(y) > \frac{k}{2}\eps^2 + t}\\
            &= \probaOf{\mathcal{L}_{\mathcal{A}}^{S \to S'}(y) > \eps'}
        \end{align*}
        Now for $\delta' \coloneqq \exp{(-\frac{t^2}{2k\eps^2})}$, we have $t = \eps\sqrt{2k\ln{(1/\delta')}}$ and $\eps' \coloneqq \frac{k}{2}\eps^2 + \eps\sqrt{2k\ln{(1/\delta')}}$.

        Hence, for any set $T \in \mathcal{Y}$:
        \begin{align*}
            \probaOf{\mathcal{A}(A(S), U) \in T}
            &\leq \probaOf{\mathcal{L}^{S \to S'}_{\mathcal{A}}(y) > \eps'} + \sum_{y \in T: \mathcal{L}^{S \to S'}_{\mathcal{A}}(y) \leq \eps'} \probaOf{\mathcal{A}(A(S), U) = y}\\
            &\leq \delta' + \sum_{y \in T: \mathcal{L}^{S \to S'}_{\mathcal{A}}(y) \leq \eps'} e^{\eps'} \probaOf{\mathcal{A}(A(S'), \emptyset) = y}\\
            &\leq \delta' + e^{\eps'} \probaOf{\mathcal{A}(A(S'), \emptyset) \in T}
        \end{align*}
        where the second inequality is from the definition of unlearning. Thus, along with an application of \cite[Lemma~1.5]{vadhan_complexity_dp}, this proves that $\mathcal{A} = (\bar{A}_1, \ldots, \bar{A}_k)$ is indeed $(\eps', \delta' + k\delta)$-unlearning w.r.t. learning algorithm $A$.
    \end{proof}

\end{document}